\documentclass{article}

\newif\iffull
\fulltrue
    \usepackage[preprint]{neurips_2022}

\setcitestyle{numbers,square}

\usepackage[utf8]{inputenc} 
\usepackage[T1]{fontenc}    
\usepackage{url}            
\usepackage{booktabs}       
\usepackage{amsfonts}       
\usepackage{nicefrac}       
\usepackage{microtype}      
\usepackage{xcolor}         

\usepackage[inline, shortlabels]{enumitem}
\usepackage[linesnumbered,vlined]{algorithm2e}

\usepackage{amsmath}
\usepackage{amssymb}
\usepackage{mathtools}
\usepackage{amsthm}
\usepackage{thmtools,thm-restate}
\usepackage[capitalize,noabbrev]{cleveref}

\newcounter{cntLemmaNumber}

\newcounter{cntTheoremNumber}

\theoremstyle{plain}
\newtheorem{theorem}{Theorem}[section]

\newtheorem{lemma}[theorem]{Lemma}
\newtheorem{claim}[theorem]{Claim}

\theoremstyle{definition}

\theoremstyle{remark}

\newcommand{\eps}{\epsilon}
\newcommand{\accfunc}{acc(W_{i,j}, \bullet)}
\newcommand{\floats}{\mathbb{F}}

\newcommand{\size}[1]{\ensuremath{\left|#1\right|}}
\newcommand{\norm}[1]{\ensuremath{\|#1\|}}
\newcommand{\set}[1]{\left\{ #1 \right\}}

\newcommand{\Gcomment}[1]{{\color{red} #1}}

\DeclarePairedDelimiterX{\infdivx}[2]{(}{)}{%
  #1\;\delimsize\|\;#2%
}
\newcommand{\KL}{D_{KL}\infdivx}
\DeclarePairedDelimiter{\floor}{\lfloor}{\rfloor}
\DeclarePairedDelimiter{\ceil}{\lceil}{\rceil}

\title{SGD Through the Lens of Kolmogorov Complexity}

\author{%
  Gregory Schwartzman \\
  JAIST\\
  \texttt{greg@jaist.ac.jp} \\
}

\begin{document}
\maketitle
\begin{abstract}
We prove that stochastic gradient descent (SGD) finds a solution that achieves $(1-\epsilon)$ classification accuracy on the entire dataset. We do so under two main assumptions: (1. Local progress) The model accuracy improves \emph{on average} over batches. (2. Models compute simple functions) The function computed by the model is simple (has low Kolmogorov complexity). It is sufficient that these assumptions hold only for a \emph{tiny fraction} of the epochs.

Intuitively, the above means that intermittent \emph{local progress} of SGD implies \emph{global progress}. Assumption 2 trivially holds for underparameterized models, hence, our work gives the first convergence guarantee for general, \emph{underparameterized models}. Furthermore, this is the first result which is completely \emph{model agnostic} - we do not require the model to have any specific architecture or activation function, it may not even be a neural network.
Our analysis makes use of the entropy compression method, which was first introduced by Moser and Tardos \cite{Moser09, MoserT10} in the context of the Lovász local lemma. 

\end{abstract}
\section{Introduction} 
Stochastic gradient descent (SGD) is at the heart of modern machine learning. However, we are still lacking a theoretical framework that explains its performance for general, non-convex functions. There are known convergence guarantees for \emph{overparameterized} neural networks \cite{DuLL0Z19, Allen-ZhuLS19, Zou2018, ZouG19}. These works show that if the number of network parameters is considerably larger than the number of data points, then the network is guaranteed to converge to a (near) global optimum. The main drawback of these results is the requirement of exceedingly large (and wide) models, where the size is a large polynomial in $n$ (the size of the dataset). In other words, the analysis is dependent on the \emph{structure} of the model. In this work, we take a different approach and show convergence guarantees based on the non-structural properties of the model.

Our two main assumptions are: (1) The model accuracy improves \emph{on average} over batches following the GD step. That is, if we look at the difference: (batch accuracy \emph{after} GD $-$ batch accuracy \emph{before} GD), \emph{averaged} over an epoch, then it is greater than some positive parameter. (2) The \emph{function} computed by the model is simple (has Kolmogorov complexity sufficiently smaller than $n$). Note that this is different than saying that the model is simple. It might be the case that a very complicated model computes an extremely simple function. A classical example is a high degree polynomial over a finite field that always evaluates to 0. 

For our convergence guarantees to hold, we only require the above assumptions to holds for some \emph{tiny fraction} of the iterations.
Note that Assumption 2 trivially holds when the model is \emph{underparametrized} (number of model parameters is sufficiently smaller than $n$). To the best of our knowledge, this is the first convergence guarantee for general, underparameterized models. Furthermore, this is the first result which is completely \emph{model agnostic}. That is, we don't require the model to have any specific architecture or activation function. It may not even be a neural network. Finally, our results answer an intriguing theoretical question: is it possible for SGD to achieve \emph{local progress} (improvement over batches), but no \emph{global progress} (good accuracy for the entire dataset)? Intuitively, one might think that this is indeed a situation that might occur. However, we show that (under Assumption 2) this cannot be the case.
\paragraph{Overview of our results} 
We consider batch SGD, where the model is initialized randomly and the dataset is shuffled once at the beginning of each epoch and then divided into batches. Note that this paper does not deal with the generalization abilities of the model, thus the dataset is always the training set. In each epoch, the algorithm goes over the batches one by one, and performs gradient descent to update the model. This is the ``vanilla" version of SGD, without any acceleration or regularization (for a formal definition see Section~\ref{sec: prelims}). Furthermore, for the sake of analysis, we add a termination condition after every GD step: if the accuracy on the entire dataset is at least $1-\epsilon$ we terminate. Thus in our case termination implies good accuracy.

We show that under reasonable assumptions, local progress of SGD implies global progress. Let $\eps \in (0,1)$ be an accuracy parameter of our choice and let $X$ be the dataset (every element in $X$ consists of a data point, a label, and some unique ID), where $\size{X}=n$. Let $K(X)$ be the Kolmogorov complexity of $X$ (the shortest Turing machine which outputs $X$). We can also define the Kolmogorov complexity of a function as the shortest Turing machine which computes the function (see Section~\ref{sec: prelims} for a formal definition).
The assumptions we make are formally stated at the beginning of Section~\ref{sec: sgd conv}. Below is an intuitive outline:
\begin{enumerate}[leftmargin=*]
    \item \label{ass: intro local prog} After every GD step the batch accuracy is improved by at least an additive $\eps \beta'$ factor (on average, per epoch), where $\beta'$ is a local progress parameter. That is, fix some epoch and let $y_i$ be the model accuracy on the $i$-th batch before the GD step and let $y'_i$ be the model accuracy after the GD step, then we require that $\frac{1}{t}\sum_{i=1}^t (y'_i - y_i) \geq \eps \beta'$ where $t$ is the number of batches in an epoch.
    
    \item \label{ass: intro kolmogorov} The Kolmogorov complexity of the function computed by the model is at most $(\eps \beta')^{3}n/512$. In other words, either the model is underparameterized, or the function computed by the model is sufficiently \emph{simple}.
    
    \item \label{ass: intro lsmooth} The loss function is differentiable and $L$-smooth. This must also hold when taking numerical stability into account. 
    
    \item \label{ass: intro batch size} The batch size is $\Omega((\beta'\eps)^{-2}\log^2 (n /(\beta'\eps) ))$. Note, the dependence on $n$ is only \emph{polylogarithmic}.

    \item \label{ass: intro element size} The size of every element in $X$ is at most polynomial in $n$. For example, if $X$ contains images, they are not excessively large compared to $n$. Furthermore, the number of model parameters is at most polynomial in $n$ (this is not implied by Assumption~\ref{ass: intro kolmogorov}, as the model can be huge but still compute a simple function).
    
\end{enumerate}

Roughly speaking, we show that if $n$ is sufficiently large and assumptions 3 to 5 always hold and assumptions 1,2 only hold for a $\gamma$-fraction of the epochs ($\gamma \in (0,1]$ need not be a constant), then SGD must achieve an accuracy of at least $(1-\eps)$ on the \emph{entire} dataset within $O(\frac{K(X)}{\gamma(\beta'\eps)^3 n})$ epochs with high probability (w.h.p)\footnote{This is usually taken to be $1-1/n^c$ for some constant $ c>1$. For our case, it holds that $c\in (1,2)$, however, this can be amplified arbitrarily by increasing the batch size by a multiplicative constant.}. Note that the number of epochs does not depend on the size of the model.
Finally, note that when the accuracy is too high the local progress assumption cannot hold anymore (we discuss this further in \iffull Section~\ref{sec: connection eps beta}\fi), and we cannot guarantee progress past this point. 

Let us briefly discuss our assumptions. As for Assumption~\ref{ass: intro local prog}, it is quite natural as the GD step attempts to minimize the loss function which acts as a proxy for the accuracy on the data. 
However, Assumption~\ref{ass: intro kolmogorov} seems at odds with the fact that we assume the model to be initialized \emph{randomly}. Intuitively, one might believe that a random initialization surely implies that Assumption~\ref{ass: intro kolmogorov} does not hold. While it is indeed the case that the \emph{model} has high Kolmogorov complexity in this case, this is not necessarily the case for the \emph{function computed by the model}. Recent literature suggests that randomly initialized models are biased towards \emph{simple functions} rather than complicated ones (see Section 7.3 in \cite{MingardPSL21} for a survey of recent work) and that the function learned by the network has low Kolmogorov complexity \cite{PerezCL19}. Furthermore, our results also generalize if we allow \emph{lossy compression}.
There is a vast literature on model compression (see \cite{nncompression} for a survey of recent literature), where the size of the model can be significantly compressed while preserving its accuracy. This also serves as a strengthening argument towards this assumption. 

Note that assumptions (3)-(5) are rather mild. Assumption~\ref{ass: intro lsmooth} (or some variation of it) is standard in the literature \cite{DuLL0Z19, Allen-ZhuLS19, Zou2018, ZouG19}, and discussions regarding numerical stability are generally omitted. Due to our use of Kolmogorov complexity we need to explicitly state that the L-smoothness assumption holds even when we consider some subset of the real numbers (i.e., floats). We discuss this in more detail in Section~\ref{sec: prelims}. 
In Assumption~\ref{ass: intro batch size} we only require a polylogarithmic dependence on $n$, which is on par with batch sizes used in practice. Furthermore, it is well observed that when the batch size is too small SGD suffers from instability \cite{deeplearningbook}. Finally, Assumption~\ref{ass: intro element size} is extremely mild, as our dependence on the size of the elements / model is only logarithmic when we apply this assumption. That is, even elements / model as large as $n^{1000}$ leave our analysis unchanged asymptotically. 


\paragraph{Outline of our techniques} To achieve our results we make use of \emph{entropy compression}, first considered by Moser and Tardos \cite{Moser09, MoserT10} to prove a constructive version of the Lovász local lemma (the name ``entropy compression" was coined by Terrence Tao \cite{tao}). Roughly speaking, the entropy compression argument allows one to bound the running time of a randomized algorithm\footnote{We require that the number of the random bits used is proportional to the execution time of the algorithm. That is, the algorithm flips coins for every iteration of a loop, rather than just a constant number at the beginning of the execution.} by leveraging the fact that a random string of bits (the randomness used by the algorithm) is computationally incompressible (has high Kolmogorov complexity) w.h.p. If one can show that throughout the execution of the algorithm, it (implicitly) compresses the randomness it uses, then one can bound the number of iterations the algorithm may execute without terminating. To show that the algorithm has such a property, one would usually consider the algorithm after executing $t$ iterations, and would try to show that just by looking at an ``execution log" of the algorithm and some set of ``hints", whose size together is considerably smaller than the number of random bits used by the algorithm, it is possible to reconstruct all of the random bits used by the algorithm. We highly encourage readers unfamiliar with this argument to take a quick look at \cite{fortnow}, which provides an excellent 1-page example.

We apply this approach to SGD with an added termination condition when the accuracy over the entire dataset is sufficiently high, thus termination in our case guarantees good accuracy.
The randomness we compress are the bits required to represent the random permutation of the data at every epoch. So indeed the longer SGD executes, the more random bits are generated. We show that under our assumptions it is possible to reconstruct these bits efficiently starting from the dataset $X$ and the model after executing $t$ epochs. The first step in allowing us to reconstruct the random bits of the permutation in each epoch is to show that under our L-smoothness assumption and a sufficiently small GD step, SGD is \emph{reversible}. That is, if we are given a model $W_{i+1}$ and a batch $B_i$ such that $W_{i+1}$ results from taking a gradient step with model $W_i$ where the loss is calculated with respect to $B_i$, then we can \emph{uniquely} retrieve $W_i$ using only $B_i$ and $W_{i+1}$. This means that if we can efficiently encode the batches used in every epoch (i.e., using less bits than encoding the entire permutation of the data), we can also retrieve all intermediate models in that epoch (at no additional cost). We prove this claim in Section~\ref{sec: prelims}.

The crux of this paper is to show that our assumptions imply that at every epoch the batches can indeed be compressed. Let us consider two toy examples that provide intuition for our analysis. First, consider the scenario where, for every epoch, exactly in the ``middle" of the epoch (after considering half of the batches for the epoch), our model achieves 100\% accuracy on the set of all elements already considered so far during the epoch (elements on the ``left") and 0\% accuracy for the remaining elements (elements on the ``right"). Let us consider a single epoch of this execution.
We can express the permutation of the dataset by encoding two sets (for elements on the left and right), and two permutations for these sets. Given the dataset $X$ and the function computed by the model at the middle of the epoch, we immediately know which elements are on the left and which are on the right by seeing if they are classified correctly or not by the model (recall that $X$ contains data points and labels and that elements have unique IDs). Thus it is sufficient to encode only two permutations, each over $n/2$ elements. While naively encoding the permutation for the entire epoch requires $\ceil{\log (n!)}$ bits, we manage to make do with $2\ceil{\log (n/2)!}$. Using Stirling's approximation ($\log (n!) = n\log (n/e) + O(\log n)$) and omitting logarithmic factors, we get that $\log (n!) - 2\log ((n/2)!) \approx n \log (n/e) - n \log (n/2e) = n$. Thus, we can save a linear number of bits using this encoding. To achieve the above encoding, we need the function computed by the model in the middle of the epoch. Assumption~\ref{ass: intro kolmogorov} guarantees that we can efficiently encode this function. 

Second, let us consider the scenario where, in every epoch, just after a single GD step on a batch we consistently achieve 100\% accuracy on the batch. Furthermore, we terminate our algorithm when we achieve 90\% accuracy on the entire dataset. Let us consider some epoch of our execution, assume we have access to $X$, and let $W_f$ be the model at the end of the epoch. Our goal is to retrieve the last batch of the epoch, $B_f\subset X$ (without knowing the permutation of the data for the epoch). A naive approach would be to simply encode the indices in $X$ of the elements in the batch (we can sort $X$ by IDs). However, we can use $W_f$ to achieve a more efficient encoding. Specifically, we know that $W_f$ achieves 100\% accuracy on $B_f$ but only 90\% accuracy on $X$. Thus it is sufficient to encode the elements of $B_f$ using a smaller subset of $X$ (the elements classified correctly by $W_f$, which has size at most $0.9\size{X}$). This allows us to significantly compress $B_f$. Next, we can use $B_f$ and $W_f$ together with the reversibility of SGD to retrieve $W_{f-1}$. We can now repeat the above argument to compress $B_{f-1}$ and so on, until we are able to reconstruct all of the random bits used to generate the permutation of $X$ in the epoch. This will again result in a linear reduction in the number of bits required for the encoding.

In our analysis, we show generalized versions of the two scenarios above. First, we show that if at any epoch the accuracy of the model on the left and right elements differs significantly (we need not split the data down the middle anymore), we manage to achieve a meaningful compression of the randomness. Next, we show that if this scenario does not happen, then together with our local progress assumption, we manage to achieve slightly better accuracy on the batch we are trying to retrieve (compared to the set of available elements). This results in a compression similar to the second scenario.
Either way, we manage to save a significant (linear in $n$) number of bits for every epoch, which allows us to use the entropy compression argument to bound our running time.

\paragraph{Related work}  An excellent survey of recent work is given in \cite{asafnoy}, which we follow below.
There has been a long line of research proving convergence bounds for SGD under various simplifying assumptions such as: linear networks \cite{AroraCGH19, AroraCH18}, shallow networks \cite{SafranS18,DuL18, Oymak}, etc. However, the most relevant results for this paper are the ones dealing with deep, overparameterized networks \cite{DuLL0Z19, Allen-ZhuLS19, Zou2018, ZouG19}. All of these works make use of NTK (Neural Tangent Kernel)\cite{JacotHG18} and show convergence guarantees for SGD when the hidden layers have width at least $poly(n,L)$ where $n$ is the size of the dataset and $L$ is the depth of the network. We note that the exponents of the polynomials are quite large.
For example, assuming a random weight initialization, \cite{Allen-ZhuLS19} require width at least $n^{24}L^{12}$ and converges to a model with an $\epsilon$ error in $O(n^6L^2 \log (1/\eps))$ steps w.h.p. The current best parameters are due to \cite{ZouG19}, which under random weight initialization, require width at least $n^{8}L^{12}$ and converge to a model with an $\epsilon$ error in $O(n^2L^2 \log (1/\eps))$ steps w.h.p. Recent works \cite{NguyenM20, Nguyen21, asafnoy} achieve a much improved width dependence (linear / quadratic) and running times, however, they require additional assumptions which make their results incomparable with the above. Specifically, in \cite{NguyenM20, Nguyen21} a special (non-standard) weight initialization is required, and in \cite{asafnoy} convergence is shown for a non-standard activation unit.

Our results are not directly comparable with the above, mainly due to assumptions \ref{ass: intro local prog} and \ref{ass: intro kolmogorov}. However, if one accepts our assumptions, our analysis has the benefit of not requiring any specific network architecture. The convergences times are also essentially incomparable due to the dependence on $K(X)$. That is, $K(X)$ can be very small if the elements of $X$ are small and similar to each other, or very large if they are very large and contain noise (e.g., high definition images with noise). 
To the best of our knowledge, this is the first (conditional) global convergence guarantee for SGD for general, non-convex underparameterized models.

    \iffull
    \paragraph{Structure of the paper} We first present some preliminary definitions in Section~\ref{sec: prelims} followed by our main results in Section~\ref{sec: sgd conv}. 
    \fi
\section{Preliminaries}
\label{sec: prelims}
We consider the following optimization problem. We are given an input (dataset) of size $n$. Let us denote $X=\set{x_i}_{i=1}^n$ (Our inputs contain both data and labels, we do not need to distinguish them for this work). 
We also associate every $x\in X$ with a unique id of $\ceil{\log n}$ bits. We often consider batches of the input $B \subset X$. The size of the batch is denoted by $b$ (all batches have the same size). We have some model whose parameters are denoted by $W \in \mathbb{R}^d$, where $d$ is the model dimension. 
We aim to optimize a goal function of the following type: $f(W) = \frac{1}{n} \sum_{x \in X} f_x(W) $, where the functions $f_x: \mathbb{R}^d \rightarrow \mathbb{R}$ are completely determined by $x \in X$. We also define for every set $A \subseteq X$: $f_A(W) = \frac{1}{|A|} \sum_{x \in A} f_x(W) $. Note that $f_X= f$.

We denote by $acc(W,A):\mathbb{R}^d \times 2^X \rightarrow \mathbb{R}[0,1]$ the accuracy of model $W$ on the set $A \subseteq X$ (where we use $W$ to classify elements from $X$). Note that for $x\in X$ it holds that $acc(W,x)$ is a binary value indicating whether $x$ is classified correctly or not.
We require that every $f_x$ is differentiable and L-smooth: $\forall W_1,W_2 \in \mathbb{R}^d, \norm{\nabla f_x(W_1) - \nabla f_x(W_2)} \leq L \norm{W_1 - W_2}$.
This implies that every $f_A$ is also differentiable and L-smooth. To see this consider the following:
\begin{align*}
& \norm{\nabla f_A(W_1) - \nabla f_A(W_2)} = \norm{ \frac{1}{|A|} \sum_{x \in A} \nabla f_x(W_1) -  \frac{1}{|A|} \sum_{x \in A} \nabla f_x(W_2)} \\ &= \frac{1}{|A|} \norm{  \sum_{x \in A} \nabla f_x(W_1) - \nabla f_x(W_2)} 
\leq \frac{1}{|A|} \sum_{x \in A} \norm{\nabla f_x(W_1) - \nabla f_x(W_2)} \leq L \norm{W_1 - W_2}    
\end{align*}
We state another useful property of $f_A$:
\begin{lemma}
\label{lem: sgd step rev}
Let $W_1, W_2 \in \mathbb{R}^d$ and $\alpha < 1/L$. For any $A\subseteq X$, if it holds that $W_1 - \alpha \nabla f_A(W_1) = W_2 - \alpha \nabla f_A(W_2)$ then $W_1=W_2$.
\end{lemma}
\begin{proof}
Rearranging the terms we get that $W_1 - W_2  = \alpha \nabla f_A(W_1) - \alpha \nabla f_A(W_2)$. Now let us consider the norm of both sides:
$\norm {W_1 - W_2}  = \norm{\alpha \nabla f_A(W_1) - \alpha \nabla f_A(W_2)} \leq \alpha \cdot L \norm {W_1 - W_2} < \norm {W_1 - W_2}$
Unless $W_1 = W_2$, the final strict inequality holds which leads to a contradiction. 
\end{proof}
The above means that for a sufficiently small gradient step, the gradient descent process is reversible. That is, we can always recover the previous model parameters given the current ones, assuming that the batch is fixed. We use the notion of reversibility throughout this paper. However, in practice we only have finite precision, thus instead of $\mathbb{R}$ we work with the finite set $\floats \subset \mathbb{R}$. Furthermore, due to numerical stability issues, we do not have access to exact gradients, but only to approximate values $\widetilde{\nabla f_A}$. For the rest of this paper, we assume these values are L-smooth on all elements in $\floats^d$. That is, 
$$\forall W_1,W_2 \in \floats^d, A\subseteq X, \norm{ \widetilde{\nabla f_A}(W_1) - \widetilde{\nabla f_A}(W_2)} \leq L \norm{W_1 - W_2}$$
This immediately implies that Lemma~\ref{lem: sgd step rev} holds even when precision is limited. Let us state the following theorem:
\begin{theorem}
\label{thm: get orig params}
Let $W_1, W_2,...,W_k \in \floats^d \subset \mathbb{R}^d$, $A_1, A_2,...,A_k \subseteq X$ and $\alpha < 1/L$. 
If it holds that $W_i = W_{i-1}- \alpha \widetilde{\nabla f_{A_{i-1}}}(W_{i-1})$, then given $A_1, A_2,...,A_{k-1}$ and $W_k$ we can retrieve $W_1$.
\end{theorem}
\begin{proof}
Given $W_k$ we iterate over all $W\in \floats^d$ until we find $W$ such that $W_k = W - \alpha \widetilde{\nabla f_{A_{i-1}}}(W)$. Using Lemma~\ref{lem: sgd step rev}, there is only a single element such that this equality holds, and thus $W=W_{k-1}$. We repeat this process until we retrieve $W_1$.
\end{proof}
We note that the assumption that L-smoothness holds for the approximate gradients on all elements in $\floats^d$ might seem a bit strong. However, if we assume that the gradients have norm bounded by a parameter $G$ (i.e. $f_A$ is $G$-Lipschitz), then it holds that $\norm{W_i - W_{i-1}} \leq \alpha G$. This means that in Theorem~\ref{thm: get orig params}, we need to only go over elements that are close to $W_i$ when retrieving $W_{i-1}$. This in turn means that we can require that the Lipschitz condition for $\widetilde{\nabla f_{A_{i-1}}}$ only holds around $W_i$.

\paragraph{SGD}
We analyze the classic SGD algorithm presented in Algorithm~\ref{alg: sgd}. One difference to note in our algorithm, compared to the standard implementation, is the termination condition when the accuracy on the dataset is sufficiently high. In practice the termination condition is not used, however, we only use it to prove that at \emph{some point} in time the accuracy of the model is indeed sufficiently large. It is also possible to check this condition on a sample of a logarithmic number of elements of $X$ and achieve a good approximation to the accuracy w.h.p. 
\RestyleAlgo{boxruled}
\LinesNumbered
\begin{algorithm}[htbp]
	\DontPrintSemicolon
	\caption{SGD}
	\label{alg: sgd}
	
	$i \gets 1$ // epoch counter\\
	$W_{1,1}$ is an initial model\\
	\While{True}
	{
	    Take a random permutation of $X$, divided into batches $\set{B_{i,j}}_{j=1}^{n/b}$\\
	    \For{j from 1 to n/b}{
	    \lIf{$acc(W_{i,j}, X) \geq (1-\eps)$}
	    {
	        Return $W_{i,j}$
	    }
	        $W_{i,j+1} \gets W_{i,j} - \alpha \nabla f_{B_{i,j}}(W_{i,j}) $\\
	    
	    }
	    $i \gets i + 1$, $W_{i, 1} \gets W_{i-1,n/b+1}$\\
	}
\end{algorithm}
\paragraph{Kolmogorov complexity}
The Kolmogorov complexity of a string $x \in \{0,1\}^*$, denoted by $K(x)$, is defined as the size of the smallest prefix Turing machine which outputs this string. We note that this definition depends on which encoding of Turing machines we use. However, one can show that this will only change the Kolmogorov complexity by a constant factor \cite{LiV19}.
Random strings have the following useful property \cite{LiV19}:
\begin{theorem}
\label{thm: kolmogorov lb}
 For an $n$ bit string $x$ chosen uniformly at random and any $c\in \mathbb{N}$ it holds that $Pr[K(x) \geq n-c] \geq 1- 1/2^c$. 
\end{theorem}

We also use the notion of conditional Kolmogorov complexity, denoted by $K(x \mid y)$. This is the length of the shortest prefix Turing machine which gets $y$ as an auxiliary input and prints $x$. Note that the length of $y$ does not count towards the size of the machine which outputs $x$. So it can be the case that $\size{x} \ll \size{y}$ but it holds that $K(x \mid y) < K(x)$. We can also consider the Kolmogorov complexity of functions. Let $g: \set{0,1}^* \rightarrow \set{0,1}^*$ then $K(g)$ is the size of the smallest Turing machine which computes the function $g$.

The following properties of Kolmogorov complexity will be of use \cite{LiV19}. Let $x,y,z$ be three strings:
\begin{itemize}
\item Extra information: $K(x \mid y,z) \leq K(x \mid z) + O(1) \leq K(x,y \mid z) + O(1)$
    \item Subadditivity: $K(xy \mid z) \leq K(x \mid z,y) + K(y \mid z) + O(1) \leq K(x \mid z) + K(y \mid z) + O(1)  $
    
\end{itemize}

\iffull
 Throughout this paper we often consider the value $K(A)$ where $A$ is a set. Here the program computing $A$ need only output the elements of $A$ (in any order). When considering $K(A \mid B)$ such that $A \subseteq B$, it holds that $K(A \mid B) \leq \ceil{\log \binom{\size{B}}{\size{A}}} + O(\log \size{B})$. To see why, consider Algorithm~\ref{alg: set}.
\RestyleAlgo{boxruled}
\LinesNumbered
\begin{algorithm}[htbp]
	\DontPrintSemicolon
	\caption{Compute $A$ given $B$ as input}
	\label{alg: set}
	$m_A \gets \size{A}, m_B \gets \size{B},i \gets 0, i_A$ is a target index\\
	\For{every subset $C\subseteq B$ s.t $\size{C} = m_A$ (in a predetermined order)}
	{
	    \lIf{$i = i_A$}{Print $C$}
	    $i \gets i+1$\\
	}
\end{algorithm}
In the algorithm $i_A$ is the index of $A$ when considering some ordering of all subsets of $B$ of size $\size{A}$. Thus $\ceil{\log \binom{\size{B}}{\size{A}}}$ bits are sufficient to represent $i_A$. The remaining variables $i, m_A, m_B$ and any additional variables required to construct the set $C$ are all of size at most $O(\log \size{B})$ and there is at most a constant number of them.

During our analysis, we often bound the Kolmogorov complexity of tuples of objects. For example, $K(A,P \mid B)$ where $A \subseteq B$ is a set and $P:A \rightarrow [\size{A}]$ is a permutation of $A$ (note that $A,P$ together form an ordered tuple of the elements of $A$). 
Instead of explicitly presenting a program such as Algorithm~\ref{alg: set}, we say that if $K(A \mid B) \leq c_1$ and $c_2$ bits are sufficient to represent $P$, thus $K(A,P \mid B) \leq c_1+c_2 + O(1)$. This just means that we directly have a variable encoding $P$ into the program that computes $A$ given $B$ and uses it in the code. For example, we can add a permutation to Algorithm~\ref{alg: set} and output an ordered tuple of elements rather than a set. Note that when representing a permutation of $A, \size{A}=k$, instead of using functions, we can just talk about values in $\ceil{\log k!}$. That is, we can decide on some predetermined ordering of all permutations of $k$ elements, and represent a permutation as its number in this ordering.

\paragraph{Entropy and KL-divergence}
Our proofs make extensive use of binary entropy and KL-divergence. In what follows we define these concepts and provide some useful properties.

\emph{Entropy}: For $p\in [0,1]$ we denote by $h(p) = -p \log p - (1-p)\log (1-p)$ the entropy of $p$. Note that $h(0)=h(1)=0$. 

\emph{KL-divergence}: For $p,q \in [0,1]$ let $\KL{p}{q} = p \log \frac{p}{q} + (1-p) \log \frac{1-p}{1-q}$ be the Kullback Leibler divergence (KL-divergence) between two Bernoulli distributions with parameters $p, q$. We also state the following useful private case of Pinsker's inequality: $\KL{p}{q} \geq 2(p-q)^2$.


\begin{restatable}{lemma}{entropyub}
\label{lem: entropy ub}
For $p\in [0,1]$ it holds that $h(p) \leq p \log (e/p)$.
\end{restatable}
\iffull
\begin{proof}
Let us write our claim as:
$$h(p) = -p \log p - (1-p)\log (1-p) \leq p \log (e/p)$$
Rearranging we get:
\begin{align*}
     &- (1-p)\log (1-p) \leq p \log p + p \log (1/p) + p\log e
     \\ & \implies - (1-p)\log (1-p) \leq  p\log e
     \\ & \implies -\ln(1-p) \leq  \frac{p}{(1-p)}  
\end{align*}

Note that $- \ln(1-p) = \int_0^{p} \frac{1}{(1-x)}dx \leq p\cdot\frac{1}{(1-p)}$.
Where in the final transition we use the fact that $\frac{1}{(1-x)}$ is monotonically increasing on $[0,1]$. This completes the proof.
\end{proof}
\fi

\begin{restatable}{lemma}{entropyformula}
\label{lem: entropy formula}
For $p,\gamma,q \in [0,1]$ where $p\gamma \leq q, (1-p)\gamma \leq (1-q)$ it holds that 
$$q h(\frac{p\gamma}{q }) + (1-q) h(\frac{(1-p)\gamma}{ (1-q)} ) \leq h(\gamma)- \gamma \KL{p}{q}$$
\end{restatable}
\iffull
\begin{proof}
Let us expand the left hand side using the definition of entropy:
\begin{align*}
     &q h(\frac{p\gamma}{q }) + (1-q) h(\frac{(1-p)\gamma}{ (1-q)} )
     \\ &=-q (\frac{p\gamma}{q }\log \frac{p\gamma}{q } + (1-\frac{p\gamma}{q })\log (1-\frac{p\gamma}{q })) 
     \\ &- (1-q)(\frac{(1-p)\gamma}{ (1-q)}\log \frac{(1-p)\gamma}{ (1-q)} + (1-\frac{(1-p)\gamma}{ (1-q)})\log (1-\frac{(1-p)\gamma}{ (1-q)}))
     \\ &=-(p\gamma\log \frac{p\gamma}{q } + (q -p\gamma)\log \frac{q -p\gamma}{q}) 
     \\ &- ((1-p)\gamma \log \frac{(1-p)\gamma}{ (1-q)} + ((1-q)-(1-p)\gamma)\log \frac{(1-q)-(1-p)\gamma} {1-q})
     \\&=-\gamma \log \gamma - \gamma \KL{p}{q}
     \\ &-(q -p\gamma)(\log \frac{q -p\gamma}{q}) 
     - ((1-q)-(1-p)\gamma)\log \frac{(1-q)-(1-p)\gamma}{1-q})
\end{align*}
Where in the last equality we simply sum the first terms on both lines.
To complete the proof we use the log-sum inequality for the last expression. The log-sum inequality states that: Let $\set{a_k}^m_{k=1}, \set{b_k}^m_{k=1}$ be \emph{non-negative} numbers and let $a=\sum^m_{k=1} a_k, b=\sum^m_{k=1} b_k$, then $\sum_{k=1}^m a_i \log \frac{a_i}{b_i} \geq a \log \frac{a}{b}$.
We apply the log-sum inequality with $m=2, a_1=q -p\gamma, a_2=(1-q)-(1-p)\gamma, a=1-\gamma$ and $b_1=q, b_2=1-q, b=1$,
getting that:
$$(q -p\gamma)(\log \frac{q -p\gamma}{q}) + ((1-q)-(1-p)\gamma)\log \frac{(1-q)-(1-p)\gamma}{1-q})\geq (1-\gamma)\log (1-\gamma)$$
Putting everything together we get that
\begin{align*}
    &-\gamma \log \gamma - \gamma \KL{p}{q}
     \\ &-(q -p\gamma)(\log \frac{q -p\gamma}{q}) 
     - ((1-q)-(1-p)\gamma)\log \frac{(1-q)-(1-p)\gamma}{1-q})
     \\ &\leq -\gamma \log \gamma - (1-\gamma)\log (1-\gamma) - \gamma \KL{p}{q} = h(\gamma)- \gamma \KL{p}{q}
\end{align*}

\end{proof}
\fi
\paragraph{Representing sets}
Let us state some useful bounds on the Kolmogorov complexity of sets.

\begin{restatable}{lemma}{bitsforset}
\label{lem: bits for set}
Let $A \subseteq B, \size{B}=m, \size{A} = \gamma m$, and let $g:B \rightarrow \set{0,1}$. For any set $Y\subseteq B$ let $Y_1 = \set{x \mid x\in Y, g(x)=1}, Y_0 = Y \setminus Y_1$ and $\kappa_Y = \frac{\size{Y_1}}{\size{Y}}$. It holds that 
$$K(A \mid B, g) \leq m(h(\gamma) -  2\gamma(\kappa_B-\kappa_A)^2)  + O(\log m) $$
\end{restatable}
\iffull
\begin{proof}
The algorithm is very similar to Algorithm~\ref{alg: set}, the main difference is that we must first compute $B_1, B_0$ from $B$ using $g$, and select $A_1, A_0$ from $B_1, B_0$, respectively, using two indices $i_{A_1}, i_{A_0}$. Finally we print $A=A_1\cup A_0$. We can now bound the number of bits required to represent $i_{A_1}, i_{A_0}$.
Note that $\size{B_1}=\kappa_B m, \size{B_0}=(1-\kappa_B) m$. Note that for $A_1$ we pick $\gamma \kappa_A m$ elements from $\kappa_B m$ elements and for $A_0$ we pick $\gamma (1-\kappa_A) m$ elements from $(1-\kappa_B) m$ elements.
The number of bits required to represent this selection is:
\begin{align*}
    & \ceil{\log \binom{\kappa_B m}{\gamma \kappa_A m}} + \ceil{\log \binom{(1-\kappa_B) m}{\gamma (1-\kappa_A) m}} 
    \leq \kappa_B m h(\frac{\gamma \kappa_A}{ \kappa_B}) + (1-\kappa_B) m h(\frac{\gamma (1-\kappa_A)}{ (1-\kappa_B)}) 
    \\ &\leq m(h(\gamma) - \gamma \KL{\kappa_B}{\kappa_A}) \leq m(h(\gamma) - 2\gamma(\kappa_B-\kappa_A)^2)
\end{align*}
Where in the first inequality we used the fact that $\forall 0\leq k \leq n, \log \binom{n}{k} \leq n h(k/n)$, Lemma~\ref{lem: entropy formula} in the second transition, and the lower bound on the KL-divergence in the third transition. The additional $O(\log m)$ factor is due to various counters and variables, similarly to Algorithm~\ref{alg: set}.
\end{proof}
\fi
\paragraph{Concentration bounds} We state the following Hoeffding bound for sampling without replacement. 
\begin{theorem}[Hoeffding \cite{hoeffding1994probability, chvatal1979tail}]
\label{thm: hoeffding no rep}
Let $\chi = (y_1,...,y_N)$ be a finite population of size $N$, such that $\forall i, y_i \in \{0,1\}$, and let $(Y_1,...,Y_k)$ be a random sample drawn without replacement from $\chi$ (first $Y_1$ is chosen uniformly at random from $\chi$, then $Y_2$ is chosen uniformly at random from $\chi \setminus \set{Y_1}$ and so on). Let $\mu = \frac{1}{N}\sum_{y \in \chi} y$.
For all $ 0 \leq \delta \leq \mu k$ it holds that: $P(\frac{1}{k}\sum_{i=1}^k Y_k - \mu \leq -\delta) \leq e^{ -2k\delta^2}$.
\end{theorem}

One important thing to note is that if we have access to $\chi$ as a shuffled tuple (or shuffled array), then taking any subset of $k$ indices, is a random sample of size $k$ drawn from $\chi$ without repetitions. 
Furthermore, if we take some sample $Y$ from the shuffled population $\chi$ and then take a second sample $Z$ (already knowing $Y$) from the remaining set $\chi \setminus Y$, then $Z$ is sampled without repetitions from $\chi \setminus Y$.

\fi
\section{Convergence guarantees for SGD}
\label{sec: sgd conv}
In this section, we prove that Algorithm~\ref{alg: sgd} eventually terminates. That is, a $(1-\eps)$ accuracy for the entire dataset is achieved. Let us start by defining some notation, followed by formal definitions of our assumptions.
\begin{figure}[h]
	\centering
	\includegraphics[width=0.9\textwidth]{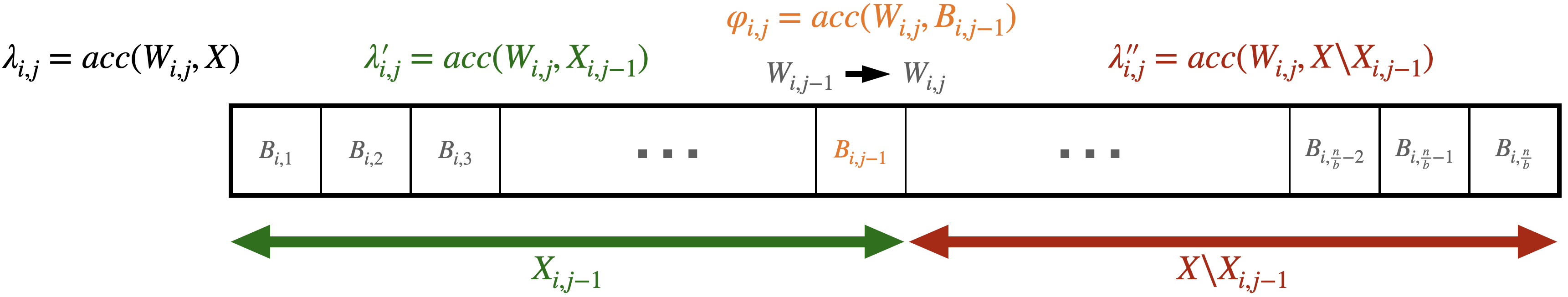}
	\caption{A visual summary of our notations.}
	\label{fig: defs}
\end{figure}

\paragraph{Notations} First let us define some useful notations:
\begin{itemize}[leftmargin=*]
    \item $\lambda_{i,j} = acc(W_{i,j}, X)$. This is the accuracy of the model in epoch $i$ on the entire dataset $X$, \emph{before} performing the GD step on batch $j$.
    \item $\varphi_{i,j}= acc(W_{i,j}, B_{i,j-1})$. This is the accuracy of the model on the $(j-1)$-th batch in the $i$-th epoch \emph{after} performing the GD step on the batch.
    \item $X_{i,j} =\bigcup_{k=1}^{j} B_{i,k}$ (note that $\forall i,X_{i,0} = \emptyset, X_{i,n/b} = X$). This is the set of elements in the first $j$ batches of epoch $i$.
    Let us also denote $n_{j} = \size{X_{i,j}} = jb$ (Note that $\forall j,i_1,i_2, \size{X_{i_1,j}} = \size{X_{i_2,j}}$, thus $i$ need not appear in the subscript).
    \item $\lambda'_{i,j}=acc(W_{i,j}, X_{i,j-1}), \lambda''_{i,j}=acc(W_{i,j}, X \setminus X_{i,j-1})$, where $\lambda'_{i,j}$ is the accuracy of the model on the set of all previously seen batch elements, \emph{after} performing the GD step on the $(j-1)$-th batch and $\lambda''_{i,j}$ is the accuracy of the same model, on all remaining elements ($j$-th batch onward). To avoid computing the accuracy on empty sets, $\lambda'_{i,j}$ is defined for $j\in [2,n/b+1]$ and $\lambda''_{i,j}$ is defined for $j\in [1,n/b]$. 
    
\end{itemize}


\paragraph{Assumptions}  In our analysis we consider $t$ epochs of the SGD algorithm, where we bound $t$ later. 
Let us state the assumptions we make (later we discuss ways to further alleviate these assumptions).

\begin{enumerate}[leftmargin=*]
    
    
    \item \label{ass: local progress} Local progress: There exists some constant $\beta'$ such that \sloppy{$\forall i\in [t], \frac{b}{n}\sum_{j=2}^{n/b+1} (\varphi_{i,j} - acc(W_{i,j-1}, B_{i,j-1})) > \beta' \epsilon$}. That is, we assume that in every epoch, the accuracy of the model on the batch after the GD step increases by an additive $\beta' \epsilon$ factor on average. Let us denote $\beta = \beta' \eps$ 
    \iffull
    (we explain the reason for this parameterization in Section~\ref{sec: connection eps beta}).
    \fi
    
    \item \label{ass: compressible models} Models compute \emph{simple} functions: $\forall i \in [t], j\in[n/b], K(\accfunc \mid X) \leq \frac{n\beta^3}{512}$ (where $\accfunc$ is the accuracy function of the model with parameters $W_{i,j}$). This is actually a slightly weaker condition than bounding the Kolmogorov complexity of the function computed by the model. 
    That is, given $X$ and the function computed by the model we can compute $\accfunc$, but not the other way around. For simplicity of notation, we assume that $W_{i,j}$ contains not only the model parameters, but also the architecture of the model and the indices $i,j$.  
    
    \item \label{ass: batch lb} $b \geq 8 \beta^{-2}\ln^2 (tn^3)$. That is, the batch must be sufficiently large. After we bound $t$, we get that $b = \Theta(\beta^{-2} \log^2 (n/\beta))$ is sufficient.
    
    \item \label{ass: lsmooth} $\forall x \in X, f_x$ is differentiable and L-smooth, even when taking numerical stability into account.
    \item \label{ass: data size ub} $K(X) = n^{O(1)}, d=n^{O(1)}$. That is, the size of the model and the size of every data point is at most polynomial in $n$. We note that the exponent can be extremely large, say $1000$, and this will only change our bounds by a constant.

\end{enumerate}

\paragraph{Bounding $t$:}
Our goal is to use the entropy compression argument to bound the running time of the algorithm. That is, we show that if our assumptions hold for a consecutive number of $t$ epochs, if $t$ is sufficiently large then our algorithm must terminate. Let $r_i$ be the string of random bits representing the random permutation of $X$ at epoch $i$. As we consider $t$ epochs, let $r=r_1r_2\dots r_t$.
Note that the number of bits required to represent an arbitrary permutation of $[n]$ is given by:
$\ceil{\log (n!)} = n\log n - n\log e + O(\log n) = n \log (n/e) + O(\log n)$.
Where in the above we used Stirling's approximation.
Thus, it holds that $|r| = t(n\log (n/e) + O(\log n))$ and according to Theorem~\ref{thm: kolmogorov lb}, with probability at least $1-1/n^2$ it holds that $K(r) \geq tn \log (n/e) - O(\log n)$.

Our goal for the rest of the section is to show that if our assumptions hold, then in every epoch we manage to compress $r_i$, which in turn results in a compression of $r$, leading to an upper bound on $t$ due to the incompressibility of $r$. 
For the rest of the section, we assume that all of our assumptions hold and that the algorithm does not terminate within the first $t$ epochs.
As discussed in the introduction, our proof consists of two cases. For the first case, we show that if at some point during the epoch $|\lambda'_{i,j}-\lambda''_{i,j}|$ is sufficiently large, then $r_i$ can be compressed. For the second case, if $|\lambda'_{i,j}-\lambda''_{i,j}|$ is small throughout the epoch and show that $r_i$ can be compressed by encoding every batch individually. 

\paragraph{Case 1: $|\lambda'_{i,j}-\lambda''_{i,j}|$ is large} 
In the introduction, we considered the case where during the epoch the model achieves 100\% accuracy on half the elements and 0\% on the other half. Let us show a generalized version of this statement, where instead of considering the model during the middle of the epoch, we consider the model after seeing a $\eta$-fraction of the data. The statement below holds for an arbitrary difference in the accuracies of the model.
Instead of considering $|\lambda'_{i,j}-\lambda''_{i,j}|$ directly, we first show an auxiliary compression claim involving $|\lambda'_{i,j}-\lambda_{i,j}|$ and $|\lambda''_{i,j}-\lambda_{i,j}|$. This will immediately allow us to derive the desired result for $|\lambda'_{i,j}-\lambda''_{i,j}|$, and will also be useful later on. 
Let us state the following claim.
\begin{restatable}{claim}{boundsubsetaccs}
\label{claim: bound subset accs}
Let $i\in[t],j\in[2,n/b-1],a\in \mathbb{R},\eta = (j-1)b/n$. It holds that:
\[
K(r_i \mid \accfunc, X) \leq \begin{cases}
        n(\log (n/e) - \eta 2a^2) + O(\log n), & \text{if } |\lambda'_{i,j} - \lambda_{i,j} |\geq a\\
        n(\log (n/e) - (1-\eta) 2a^2) + O(\log n), & \text{if } |\lambda''_{i,j} - \lambda_{i,j} |\geq a
\end{cases}
\]
\end{restatable}
\iffull
\begin{proof}
To encode $r_i$ we encode it as 2 separate strings, split at the $((j-1)b)$-th index. To encode each such string it is sufficient to encode the sets $X_{i,j-1}$ and $X \setminus X_{i,j-1}$, accompanied by a permutation for each set. As $X$ is given, we can immediately deduce $r_i$ given an ordering of the elements of $X$. Note that knowing $X_{i,j-1}$ and $X$ we immediately know $X\setminus X_{i,j-1}$. Thus it is sufficient to only encode $X_{i,j-1}$. Finally, we need to encode two permutations for the two sets. The number of bits required for the two permutations is bounded by:
\begin{align*}
&\ceil{\log (\eta n)!} + \ceil{\log ((1-\eta) n)!} \leq \eta n\log \frac{\eta n}{e} + (1-\eta) n\log \frac{(1-\eta) n}{e} + O(\log n)
\\ &=  n(\log \frac{n}{e} + \eta \log \eta + (1-\eta) \log (1-\eta)) + O(\log n) = n\log \frac{n}{e} - nh(\eta) + O(\log n)
\end{align*}
To encode the set $X_{i,j-1}$ we use Lemma~\ref{lem: bits for set} where $A=X_{i,j-1},B=X$ and $g(x)=acc(W_{i,j}, x)$. 
We get that
\[
K(X_{i,j-1} \mid X, \accfunc)
\leq  n(h(\eta) -2(\lambda'_{i,j}-\lambda_{i,j})^2) + O(\log n) \leq n(h(\eta) -\eta 2a^2) + O(\log n)
\]
Summing the two expressions we get that
\[
K(r_i \mid \accfunc, X) \leq n\log (n/e) - \eta a\cdot \KL{\lambda'_{i,j}}{\lambda_{i,j}}+ O(\log n) \leq n(\log (n/e) - \eta 2a^2)+ O(\log n)
\]
Encoding $X \setminus X_{i,j-1}$ instead of $X_{i,j-1}$ and using an identical argument, we get a similar bound for $\lambda''_{i,j}$.
\end{proof}
\fi
Note that in the above if $\eta = (j-1)b/n$ is too small or too large, we are unable to get a meaningful compression. Let us focus our attention on a sub-interval of $j\in[2,n/b-1]$ where we get a meaningful bound. Let us denote by $j_s = \ceil{\beta n/8b} +1, j_f=\floor{(1-\beta/8) n/b} +1$ the start and finish indices for this interval.
We state the following Corollary which completes the proof of the first case.
\begin{restatable}{corollary}{boundsubsetaccdiff}
\label{col: bound subset acc diff}
If it holds for some epoch $i$, and some $j\in [j_s, j_f]$ that $|\lambda'_{i,j} - \lambda''_{i,j} |\geq \beta/4$, then $K(r_i \mid \accfunc, X) \leq n(\log (n/e) - \frac{1}{256}\beta^3) + O(\log n)$.
\end{restatable}
\iffull
\begin{proof}
Assume w.l.o.g that $\lambda'_{i,j} \leq \lambda''_{i,j}$. Note that it holds that $\lambda_{i,j} \in [\lambda'_{i,j}, \lambda''_{i,j}]$. This is because we can write $\lambda_{i,j} = q\lambda'_{i,j}+ (1-q)\lambda''_{i,j}$ for some $ q\in [0,1]$. Thus if $|\lambda'_{i,j} - \lambda''_{i,j} |\geq \beta/4$ then either $|\lambda_{i,j} - \lambda''_{i,j} |\geq \beta/8$ or $|\lambda_{i,j} - \lambda'_{i,j} |\geq \beta/8$. 
We apply Claim~\ref{claim: bound subset accs} for either $\lambda'_{i,j}$ or $\lambda''_{i,j}$ with parameter $a=\beta/8$. Noting that $j\in [j_s, j_f]$ we get that $\eta, (1-\eta) \geq \beta/8$. Putting everything together we get:
$K(r_i \mid \accfunc, X) \leq n(\log (n/e) - \frac{1}{256}\beta^3) + O(\log n)$
\end{proof}
\fi
\paragraph{Generalizing Assumption~\ref{ass: compressible models}} The assumption is only required for this case of the proof, and the proofs of Claim~\ref{claim: bound subset accs} and Crollary~\ref{col: bound subset acc diff} still go through even if we only have a sufficiently good \emph{approximation} to $\accfunc$. That is, there exists a function $g': X \rightarrow \set{0,1}$, which achieves a sufficiently close approximation to the accuracy of $\accfunc$ on $X \setminus X_{i,j}$ and $X_{i,j}$ (say up to an additive $\beta /100$ factor). Thus, we only need the existence of a compressible function which \emph{approximates} $\accfunc$ on $X \setminus X_{i,j}$ and $X_{i,j}$.
\paragraph{Case 2: $|\lambda'_{i,j} - \lambda''_{i,j}|$ is small}  Let us now consider the second case, where $|\lambda'_{i,j} - \lambda''_{i,j}|$ is small throughout the epoch (more precisely, when $j\in[j_s, j_f]$). Roughly speaking, we would like to get an effect similar to that seen in the introduction, when we assumed we can achieve 100\% accuracy for every batch. The reason we managed to achieve compression in that case, was that we could use the model to reduce the set of potential elements the batch can be taken from when going backwards and reconstructing the batches. We will show that we can achieve a similar effect here. 
Assume we have the entire dataset, $X$, and the model which resulted from a GD step on the $(j-1)$-th batch. The argument is inductive, thus let us assume that we have reconstructed all batches $B_{i,k}, k\geq j$. This means we also know $ X \setminus X_{i,j-1}$. Knowing $X$ and $ X \setminus X_{i,j-1}$, we immediately know $X_{i,j-1}$, the set of all elements which can appear in $B_{i,j-1}$. We aim to show that the local progress assumption (Assumption~\ref{ass: local progress}) implies that (on average) $W_{i,j}$ achieves a slightly higher accuracy on $B_{i,j-1}$ than on $X_{i,j-1}$, which will allow us to achieve a meaningful compression.

Let us first show that if the batch is large enough, the accuracy of $W_{i,j}$ on $B_{i,j}$ is sufficiently close to $\lambda''_{i,j}$ (the accuracy of $W_{i,j}$ on $X \setminus X_{i,j-1}$). We state the following claim:
\begin{restatable}{claim}{batchacclb}
\label{claim: batch acc lb}
It holds w.h.p that: $\forall i\in [t], j\in[2,n/b-1], acc(W_{i,j}, B_{i,j}) \geq \lambda''_{i,j} - \beta/4$
\iffull
\begin{align*}
    &\forall i\in [t], j\in[2,n/b-1], acc(W_{i,j}, B_{i,j}) \geq \lambda''_{i,j} - \beta/4
\end{align*}
\fi
\end{restatable}
\iffull
\begin{proof}

First let us note that if $\lambda''_{i,j} \leq \beta/4$, then the claim holds trivially: $acc(W_{i,j}, B_{i,j}) \geq 0 \geq \lambda''_{i,j} - \beta/4$. So for the rest of the proof let us assume that $\lambda''_{i,j} \geq \beta/4$.

 Let us denote by $\set{Y_\ell}_{\ell=1}^b$, the set of binary indicator variables indicating whether the $\ell$-th element in $B_{i,j}$ was classified correctly by $W_{i,j}$. As stated before, we can say that elements in $B_{i,j}$ are chosen uniformly at random without repetitions from $X \setminus X_{i,j-1}$. Furthermore, it is important to note that $B_{i,j}$ is independent of $W_{i,j}$. We can imagine that $B_{i,j}$ is selected \emph{after} $W_{i,j}$ is determined.
 We can write $acc(W_{i,j}, B_{i,j}) = \frac{1}{b}\sum_{\ell=1}^b Y_\ell$. Applying a Hoeffding bound with parameters $ \mu = \lambda''_{i,j}, k=b, \delta = \beta/4 \geq \mu k$, we get that:
$$Pr[acc(W_{i,j}, B_{i,j})  \leq \lambda''_{i,j} - \beta/4] \leq e^{-2k\delta^2}= e^{-b\frac{\beta^2}{8}}$$

According to Assumption~\ref{ass: batch lb} it holds that $b \geq \frac{8 \ln (tn^3)}{\beta^2}$. Thus, we get that: $$\forall i\in [t],j\in[2,n/b-1], Pr[acc(W_{i,j}, B_{i,j})  < \lambda''_{i,j} - \beta/4] \leq 1/(n^3t)$$
We take a union bound over all bad events (at most $tn$) to get that with probability at least $1-1/n^2$ none of the bad events occur throughout the $t$ epochs of the execution. 

\end{proof}
\fi
That is, w.h.p the accuracy of $W_{i,j-1}$ on $B_{i,j-1}$ is at least its accuracy on $X \setminus X_{i,j-1}$, up to an additive $\beta/4$ factor. Next we use the fact that $\forall \ell\in [j_s,j_f],|\lambda'_{i,\ell} - \lambda''_{i,\ell}|$ is small together with the claim above to show that the accuracy of $W_{i,j-1}$ on $B_{i,j-1}$ is also not too far from the accuracy of $W_{i,j-1}$ on $X_{i,j-1}$. Finally, using the fact that on average $W_{i,j}$ has a sufficiently larger accuracy on $B_{i,j}$ compared to $W_{i,j-1}$ we prove the following claim. Note that this is an auxiliary claim, and the quantity $\sum_{j=2}^{n/b+1} (\varphi_{i,j} - \lambda'_{i,j})^2$ will appear later as the number of bits we manage to save per epoch, by using the model at every epoch to encode the batches. 
\begin{restatable}{claim}{sumofsquareslb}
\label{claim: sum of squares lb}
 If $\forall i \in [t], j\in[j_s,j_f], \size{\lambda'_{i,j} - \lambda''_{i,j}} \leq \beta/4$ then w.h.p:
 $\forall i\in [t], \sum_{j=2}^{n/b+1} (\varphi_{i,j} - \lambda'_{i,j})^2 \geq \frac{n\beta^2}{25b}$.
\end{restatable}
\iffull
\begin{proof}
Recall that $\varphi_{i,j}= acc(W_{i,j}, B_{i,j-1})$. This is the accuracy of the model on the $(j-1)$-th batch in the $i$-th epoch \emph{after} performing the GD step on the batch. Let us consider the following inequality:
\begin{align*}
    &1 \geq \lambda''_{i,j_f} -\lambda''_{i,j_s} = \sum_{j=j_s+1}^{j_f} (\lambda''_{i,j} - \lambda''_{i,j-1}) 
    \\ &= \sum_{j=j_s+1}^{j_f} (\lambda''_{i,j} - \lambda''_{i,j-1} + \varphi_{i,j} - \varphi_{i,j} )
    = \sum_{j=j_s+1}^{j_f} (\lambda''_{i,j} - \varphi_{i,j}) +\sum_{j=j_s+1}^{j_f} (\varphi_{i,j} - \lambda''_{i,j-1})   
\end{align*}

Using Claim~\ref{claim: batch acc lb} we know that w.h.p:
$$\sum_{j=j_s+1}^{j_f} \lambda''_{i,j-1} \leq  \sum_{j=j_s+1}^{j_f} (acc(W_{i,j-1}, B_{i,j-1}) + \frac{\beta}{4}) \leq  \sum_{j=j_s+1}^{j_f} acc(W_{i,j-1}, B_{i,j-1}) + \frac{\beta n}{4b}$$
So we can write:
\begin{align*}
&\sum_{j=j_s+1}^{j_f} (\varphi_{i,j} - \lambda''_{i,j-1}) \geq \sum_{j=j_s+1}^{j_f} (\varphi_{i,j} - acc(W_{i,j-1}, B_{i,j-1})) - \frac{\beta n}{4b} 
\\ &= \sum_{j=2}^{n/b+1} (\varphi_{i,j} - acc(W_{i,j-1}, B_{i,j-1})) - \sum_{j\in [2,j_s]\cup[j_f+1, n/b+1]} (\varphi_{i,j} - acc(W_{i,j-1}, B_{i,j-1})) - \frac{\beta n}{4b}
\\&\geq \sum_{j=2}^{n/b+1} (\varphi_{i,j} - acc(W_{i,j-1}, B_{i,j-1})) - \frac{\beta n}{2b} \geq 
\frac{\beta n}{b} - \frac{\beta n}{2b} = \frac{\beta n}{2b}  
\end{align*}

Where in the last line we use Assumption~\ref{ass: local progress} together with the fact that $(\varphi_{i,j} - acc(W_{i,j-1}, B_{i,j-1})) \in [-1,1]$ and 
\begin{align*}
    & \size{[2,j_s]} + \size{[j_f +1, n/b+1]} = j_s-1 + (n/b+1 - (j_f+1) + 1) = j_s + n/b -j_f   
    \\ &= \ceil{(\beta/8)\frac{n}{b}}+1 + (n/b - \floor{\frac{(1-\beta/8) n}{b}}-1))
    \leq (\beta/8)\frac{n}{b}+2+ (n/b - \frac{(1-\beta/8) n}{b}-2) \leq \frac{\beta n}{4b}
\end{align*}

Combining everything together we get that:
$$1 \geq \sum_{j=j_s+1}^{j_f} (\lambda''_{i,j} - \varphi_{i,j})  +\frac{\beta n}{2b} \implies \sum_{j=j_s+1}^{j_f} (\varphi_{i,j} - \lambda''_{i,j}) \geq \frac{\beta n}{2b} - 1$$

Let us use the above to derive a lower bound for $\sum_{j=j_s+1}^{j_f} (\lambda'_{i,j} - \varphi_{i,j})$. Let us write:
\begin{align*}
    &\sum_{j=j_s+1}^{j_f} (\lambda''_{i,j} - \varphi_{i,j}) = \sum_{j=j_s+1}^{j_f} (\lambda'_{i,j} - \varphi_{i,j}) + \sum_{j=j_s+1}^{j_f} (\lambda''_{i,j} - \lambda'_{i,j}) 
    \\ &\implies \sum_{j=j_s+1}^{j_f} (\lambda'_{i,j} - \varphi_{i,j}) \geq \frac{\beta n}{2b} - 1  -  \sum_{j=j_s+1}^{j_f} (\lambda''_{i,j} - \lambda'_{i,j}) \geq \frac{\beta n}{4b} - 1 \geq  \frac{\beta n}{5b}
\end{align*}

Where in the second to last transition we use the assumption that $\forall i \in [t], j\in[j_s,j_f], \size{\lambda'_{i,j} - \lambda''_{i,j}} \leq \beta/4$, and in the final transition we assume $n$ is sufficiently large.
Taking the square of both sides we get that:
$(\sum_{j=j_s+1}^{j_f}   \lambda'_{i,j} - \varphi_{i,j})^2 \geq (\frac{\beta n}{5b})^2$.

Finally, applying Cauchy-Schwartz inequality, we get that:
$$\sum_{j=2}^{n/b+1} (\varphi_{i,j} - \lambda'_{i,j})^2 \geq \sum_{j=j_s+1}^{j_f} (\varphi_{i,j} - \lambda'_{i,j})^2 \geq  (j_f-j_s)(\sum_{j=j_s+1}^{j_f} \varphi_{i,j} - \lambda'_{i,j})^2 \geq \frac{b}{(1-\beta/4)n}\cdot (\frac{\beta n}{5b})^2 \geq \frac{n\beta^2}{25b}$$

\end{proof}
\fi
To complete the proof for the second case we use the above claim to show that we can significantly compress all batches, this, in turn, will result in a compression of the entire permutation. We state the following claim. 
\begin{restatable}{claim}{boundkaccdiffsmall}
\label{claim: bound k acc diff small}
It holds that if $\forall i \in [t], j\in[j_s,j_f], \size{\lambda'_{i,j} - \lambda''_{i,j}} \leq \beta/4$ then w.h.p: 
$ K(r_i \mid W_{i+1,1}, X) \leq n\log \frac{n}{e} - \frac{2n\beta^2}{25} + o(n)  $
\end{restatable}
\iffull
\begin{proof}
Recall that $B_{i,j}$ is the $j$-th batch in the $i$-th epoch, and let $P_{i,j}$ be a permutation of $B_{i,j}$ such that the order of the elements in $B_{i,j}$ under $P_{i,j}$ is the same as under $r_i$. Note that given $X$, if we know the partition into batches and all permutations, we can reconstruct $r_i$. According to Theorem~\ref{thm: get orig params}, given $W_{i,j}$ and $B_{i,j-1}$ we can compute $W_{i,j-1}$. Let us denote by $Y$ the encoding of this procedure. To implement $Y$ we need to iterate over all possible vectors in $\floats^d$ and over batch elements to compute the gradients. To express this program we require auxiliary variables of size at most $O(\log \min \set{d, b}) = O(\log n)$. Thus it holds that $K(Y) = O(\log n)$.
Let us abbreviate $B_{i,1}, B_{i,2},...,B_{i,j}$ as $ (B_{i,k})_{k=1}^{j}$. We write the following.
\begin{align*}
    &K(r_i \mid X, W_{i+1,1}) \leq K(r_i, Y \mid X, W_{i+1,1}) + O(1) \leq K(r_i \mid X, W_{i+1,1}, Y) + K(Y \mid X, W_{i+1,1})+ O(1) 
    \\&\leq  O(\log n) + K((B_{i,k})_{k=1}^{n/b} \mid X, W_{i+1,1}, Y) + \sum_{j=1}^{n/b} K(P_{i,j}) 
\end{align*}

Let us bound $K((B_{i,k})_{k=1}^{n/b} \mid X, W_{i+1,1}, Y)$ by repeatedly using the subadditivity and extra information properties of Kolmogorov complexity.  
\begin{align*} 
&K((B_{i,k})_{k=1}^{n/b} \mid X, Y, W_{i+1,1}) \leq K(B_{i,n/b} \mid X, W_{i+1,1}) + K((B_{i,k})_{k=1}^{n/b-1} \mid X,Y, W_{i+1,1}, B_{i,n/b}) + O(1)
    \\ &\leq K(B_{i,n/b} \mid X, W_{i+1,1}) + K((B_{i,k})_{k=1}^{n/b-1} \mid X, Y, W_{i,n/b}, B_{i,n/b}) + O(1)
    \\ &\leq  
    K(B_{i,n/b} \mid X, W_{i+1,1}) + 
    K(B_{i,n/b-1} \mid X, W_{i,n/b}, B_{i,n/b})
    \\&+
    K((B_{i,k})_{k=1}^{n/b-2} \mid X, Y, W_{i,n/b-1}, B_{i,n/b}, B_{i,n/b-1}) + O(1)
    \\& \leq...\leq
    O(\frac{n}{b}) + \sum_{j=1}^{n/b} K(B_{i,j} \mid X, W_{i,j+1},
    (B_{i,k})_{k=j+1}^{n/b})
    \leq O(\frac{n}{b}) + \sum_{j=1}^{n/b} K(B_{i,j} \mid  X_{i,j}, W_{i,j+1})
\end{align*}

Where in the transitions we used the fact that given $W_{i,j},B_{i,j-1}$ and $Y$ we can retrieve $W_{i,j-1}$. That is, we can always bound $K(...\mid Y,W_{i,j},B_{i,j-1},...)$ by $K(...\mid Y,W_{i,j-1},B_{i,j-1},...) + O(1)$. 

To encode the order $P_{i,j}$ inside each batch, $b\log (b/e) + O(\log b)$ bits are sufficient. Finally we get that: $K(r_i \mid X, W_{i+1,1}) \leq O(\frac{n}{b}) +  \sum_{j=1}^{n/b} [K(B_{i,j} \mid  X_{i,j}, W_{i,j+1}) + b\log (b/e) + O(\log b)]$.

Let us now bound $K(B_{i,j-1} \mid X_{i,j-1}, W_{i,j})$. 
Knowing $X_{i,j-1}$ we know that $B_{i,j-1} \subseteq X_{i,j-1}$. Thus we need to use $W_{i,j}$ to compress $B_{i,j-1}$. Applying Lemma~\ref{lem: bits for set} with parameters $A=B_{i,j-1}, B=X_{i,j-1}, \eta = b/n_{j-1}, \kappa_{A}=\varphi_{i,j}, \kappa_{B}=\lambda'_{i,j}$ and $g(x)=acc(W_{i,j}, x)$. We get the following:
\begin{align*}
&K(B_{i,j-1} \mid X_{i,j-1}, W_{i,j}) 
\leq n_{j-1}(h(\frac{b}{n_{j-1}}) - \frac{2b}{n_{j-1}}(\varphi_{i,j}-\lambda'_{i,j})^2)+ O(\log n_{j-1}) \\ &= n_{j-1}h(\frac{b}{n_{j-1}}) - 2b(\varphi_{i,j}-\lambda'_{i,j})^2 + O(\log n_{j-1})    
\end{align*}

Adding $b\log (b/e) + O(\log b)$ to the above, we get the following bound on every element in the sum: 
\begin{align*}
     &n_{j-1}h(\frac{b}{n_{j-1}}) - 2b(\varphi_{i,j}-\lambda'_{i,j})^2 + b\log (b/e) + O(\log b)+ O(\log n_{j-1})
    \\ &\leq b \log \frac{e n_{j-1}}{b}+ b\log (b/e) + O(\log n_{j-1}) - 2b(\varphi_{i,j}-\lambda'_{i,j})^2  
    \\&= b\log n_{j-1} -2b(\varphi_{i,j}-\lambda'_{i,j})^2 + O(\log n_{j-1}) 
\end{align*}

Where in the second line we use Lemma~\ref{lem: entropy ub}. 
Note that the most important term in the sum is $-2b(\varphi_{i,j}-\lambda'_{i,j})^2$. That is, the more the accuracy of $W_{i,j}$ on the batch, $B_{i,j-1}$, differs from the accuracy of $W_{i,j}$ on the set of elements containing the batch, $X_{i,j-1}$, we can represent the batch more efficiently. Let us now bound the sum: $\sum_{j=2}^{n/b+1} [b\log n_{j-1} -2b(\varphi_{i,j}-\lambda'_{i,j})^2 + O(\log n_{j-1})] $.
Let us first bound the sum over $b \log n_{j-1}$:
\begin{align*}
    &\sum_{j=2}^{n/b+1} b \log n_{j-1} = \sum_{j=1}^{n/b} b \log jb 
    = \sum_{j=1}^{n/b} b (\log b + \log j)
    \\&= n\log b + b\log (n/b)!
    = n\log b + n\log \frac{n}{b\cdot e} + O(\log n) = n\log \frac{n}{e} + O(\log n)
\end{align*}
Applying Claim~\ref{claim: sum of squares lb} we get that $2b\sum_{j=2}^{n/b+1} (\varphi_{i,j}-\lambda'_{i,j})^2 \geq \frac{2n\beta^2}{25}$.
Finally we can write that:
\begin{align*}
&K(r_i \mid X, W_{i+1,1}) \leq O(\frac{n}{b}) + \sum_{j=2}^{n/b+1} [b \log n_{j-1} -2b(\varphi_{i,j}-\lambda'_{i,j})^2 +O(\log n)] \\&\leq n\log \frac{n}{e} - \frac{2n\beta^2}{25} + \frac{n}{b}\cdot O(\log n) = n\log \frac{n}{e} - \frac{2n\beta^2}{25} + o(n)    
\end{align*}
Where $\frac{n}{b}\cdot O(\log n)=o(n)$ because $b = \omega(\log n)$.
\end{proof}
\fi

This completes the analysis of the second case. As for every epoch either case 1 or case 2 hold, we can conclude that we always manage to compress $r_i$. Combining the claim above with Corollary~\ref{col: bound subset acc diff} we get:
\begin{restatable}{corollary}{riub}
\label{col: ri ub}
It holds w.h.p that
$\forall i\in [t], K(r_i \mid X, W_{i+1,1}) \leq n(\log (n/e) - \frac{1}{512}\beta^3) + O(\log n)$.
\end{restatable}
\iffull
\begin{proof}
    In every epoch either the conditions of Corollary~\ref{col: bound subset acc diff} or Claim~\ref{claim: bound k acc diff small} hold. If the conditions of Corollary~\ref{col: bound subset acc diff} hold for some $j\in [j_s,j_f]$ then we can write:
    \begin{align*}
    &K(r_i \mid X, W_{i+1,1}) \leq K(r_i \mid X, W_{i+1,1}, \accfunc) + K(\accfunc \mid X) + O(1)
    \\ &\leq \frac{n}{512}\beta^3 + n(\log (n/e) - \frac{1}{256}\beta^3) + O(\log n) \leq n(\log (n/e) - \frac{1}{512}\beta^3) + O(\log n)    
    \end{align*}
    Where we use Assumption~\ref{ass: compressible models} to bound $K(\accfunc \mid X)$. If the conditions of Claim~\ref{claim: bound k acc diff small} hold, then we immediately get:
    $$K(r_i \mid X, W_{i+1,1}) \leq n\log \frac{n}{e} - \frac{2n\beta^2}{25} + o(n) \leq n(\log (n/e) - \frac{1}{512}\beta^3) + O(\log n)$$
\end{proof}
\fi
We use the fact that for every epoch that the algorithm does not terminate we manage to significantly compress $r_i$ together with the fact that $r$ is incompressible to bound $t$. We state the following claim:

\begin{restatable}{claim}{boundt}
It holds w.h.p that $t \leq O(\frac{K(X)}{n \beta^3})$.
\end{restatable}

\begin{proof}
\iffull
 Similarly to the definition of $Y$ in Claim~\ref{claim: bound k acc diff small}, let $Y'$ be the program which receives $X,r_i, W_{i+1,1}$ as input and repeatedly applies Theorem~\ref{thm: get orig params} to retrieve $W_{i,1}$. As $Y'$ just needs to reconstruct all batches from $X, r_i$ and call $Y$ for $n/b$ times, it holds that $K(Y') = O(\log n)$.
 \fi
Using the subadditivity and extra information properties of $K()$, together with the fact that $W_{1,1}$ can be reconstructed given $X,W_{t+1,1}, Y'$, we write the following: 
\begin{align*}
    &K(r, W_{1,1}) \leq  K(r,W_{1,1}, X,Y', W_{t+1,1})+ O(1) \leq K(W_{1,1,},X,W_{t+1,1},Y') + K(r\mid X,Y', W_{t+1,1}) + O(1)
    \\ &\leq K(X,W_{t+1,1}) + K(r\mid X,Y', W_{t+1,1}) + O(\log n)\leq K(X) + d + K(r\mid X,Y', W_{t+1,1}) + O(\log n) 
\end{align*}

First we note that:
$\forall i \in [t-1], K(r_{i}\mid X,Y', W_{i+2,1}, r_{i+1}) \leq K(r_{i}\mid X,Y', W_{i+1,1}) + O(1)$.
Where in the last inequality we simply execute $Y'$ on $X, W_{i+2,1}, r_{i+1}$ to get $W_{i+1,1}$.
 Let us write:
\begin{align*}
    &K(r_1r_2 \dots r_{t}\mid X,Y', W_{t+1,1}) \leq  K(r_t \mid X,Y', W_{t+1,1}) + K(r_1r_2 \dots r_{t-1}\mid X,Y', W_{t+1,1}, r_t) + O(1)
    \\ &\leq K(r_t \mid X, W_{t+1,1}) + K(r_1r_2 \dots r_{t-1}\mid X,Y', W_{t,1})+ O(1)
    \\ & \leq K(r_t \mid X, W_{t+1,1}) + K(r_{t-1} \mid X, W_{t,1}) +  K(r_1r_2 \dots r_{t-2}\mid X,Y', W_{t-1,1}) + O(1)
    \\& \leq \dots \leq O(t) + \sum^t_{k=1} K(r_k \mid X, W_{k+1,1}) \leq t[n(\log (n/e) - \frac{1}{512}\beta^3) + O(\log n)]
\end{align*}

Where the last inequality is due to Corollary~\ref{col: ri ub}.
Combining everything together we get that: $K(r) \leq K(X,W_{t+1,1}) + t[n(\log (n/e) - \frac{1}{512}\beta^3) + O(\log n)]$.

Our proof implies that we can reconstruct not only $r$, but also $W_{1,1}$ using $X, W_{t+1,1}$. Due to the incompressibility of random strings, we get that w.h.p $K(r, W_{1,1}) \geq d + tn\log (n/e) - O(\log n)$. 
Combining the lower and upper bound for $K(r, W_{1,1})$ we can get the following inequality, using which we can bound $t$. 
\begin{align*}
  &d + tn\log (n/e) - O(\log n)\leq K(X) +d + t[n(\log (n/e) - \frac{1}{512}\beta^3) + O(\log n)]
  \\ &\implies t(\frac{n}{512}\beta^3 - O(\log n)) \leq K(X)
  \implies t = O(\frac{K(X)}{n\beta^3}) \qedhere
\end{align*}
\end{proof}

\paragraph{Assumptions \ref{ass: local progress} and \ref{ass: compressible models} need not always hold}  Let us consider the case where both Assumption~\ref{ass: local progress} and Assumption~\ref{ass: compressible models} only hold for some fraction $\gamma$ of the $t$ epochs that we analyze. Let us denote by $J_g \subset [t]$ the set of all good epochs where the assumptions hold (the rest of the epochs are bad epochs).
This means that we can no longer achieve compression for bad epochs. However, for every bad epoch $i$, we can simply encode $r_i$ directly. Given $r_i$ we can reconstruct all of the batches, and retrieve $W_{i,1}$ from $W_{i+1,1}$ using Theorem~\ref{thm: get orig params}.  We can write that:
\begin{align*}
    K(r_i \mid X, W_{i+1,1}) \leq
\begin{cases}
    n(\log (n/e) - \frac{1}{512}\beta^3) + O(\log n),& \text{if } i\in J_g\\
    n(\log (n/e)) + O(\log n),              & \text{otherwise}
\end{cases}
\end{align*}

This allows us to bound $t$ as follows:
\begin{align*}
  &d + tn\log (n/e) - O(\log n)
  \\&\leq K(X)+ d + \gamma t[n(\log (n/e) - \frac{1}{512}\beta^3) + O(\log n)] + (1-\gamma) t[n(\log (n/e) ) + O(\log n)]
  \\ &\implies t(\gamma \frac{n}{512}\beta^3 - O(\log n)) \leq K(X) 
 \implies t = O(\frac{K(X)}{n\gamma \beta^3})
\end{align*}

Note that $t$ also appears in the batch size. By Assumption~\ref{ass: data size ub}, we know that $t=O(\beta^{-3} poly(n))$, which in turn means that $b = \Theta(\beta^{-2} \log^2 (n/\beta))$ is a sufficient condition for Assumption~\ref{ass: data size ub} to hold. 
We state our main theorem for this section:
\begin{theorem}
If assumptions (3)-(5) hold throughout the execution of SGD and assumptions (1)-(2) hold for an $\gamma$-fraction of the epochs, then SGD achieves an accuracy of $(1-\eps)$ on $X$ in $O(\frac{K(X)}{\gamma n (\eps \beta')^3})$ epochs w.h.p.
\end{theorem}
\iffull
\subsection{Connection between $\epsilon$ and $\beta$} 
\label{sec: connection eps beta}
We note that the local progress assumption does not make sense when the accuracy is too close to 1. That is, if the accuracy is too high, there is no room for improvement. We formalize this intuition in the following claim:
\begin{claim}
If for some $i$ it holds that $\forall j, \lambda_{i,j}\geq (1-\eps)$, then $\beta \leq \frac{4}{3}\eps \ln (n/b)$.
\end{claim}
\begin{proof}

Recall the local progress assumption: $\frac{b}{n}\sum_{j=2}^{n/b+1} (\varphi_{i,j} - acc(W_{i,j-1}, B_{i,j-1})) > \beta$. Applying Claim~\ref{claim: batch acc lb} we know that w.h.p $acc(W_{i,j-1}, B_{i,j-1}) \geq \lambda''_{i,j-1} -\beta/4$. So we can write:
$$\beta < \frac{b}{n}\sum_{j=2}^{n/b+1} (\varphi_{i,j} - \lambda''_{i,j-1} + \beta/4) \leq \beta/4 + \frac{b}{n}\sum_{j=2}^{n/b+1} (1 - \lambda''_{i,j-1})$$


It holds that $\lambda_{i,j} \geq (1-\eps)$. Let us now bound $\lambda''_{i,j}$ using $\lambda_{i,j}$. We can write the following for $j\in[2,n/n+1]$:
\begin{align*}
    & n\lambda_{i,j} = \sum_{x\in X} acc(W_{i,j}, x) = \sum_{x\in X_{i,j-1}} acc(W_{i,j}, x) + \sum_{x\in X\setminus X_{i,j-1}} acc(W_{i,j}, x) 
    \\&= (j-1)b \lambda'_{i,j} + (n-(j-1)b) \lambda''_{i,j}
    \\ &\implies \lambda''_{i,j} = \frac{n\lambda_{i,j} - (j-1)b\lambda'_{i,j}}{n-(j-1)b} 
    \geq \frac{n(1-\eps) - (j-1)b}{n-(j-1)b}= 1 - \frac{\eps n}{n-(j-1)b}
\end{align*}
Plugging the above into the first inequality we get:
\begin{align*}
    &\frac{3\beta}{4} \leq 
    \frac{b}{n}\sum_{j=2}^{n/b+1} \frac{\eps n}{n-(j-2)b } \leq \sum_{j=2}^{n/b+1} \frac{\eps }{n/b-(j-2) } \leq \eps (1+\ln (n/b)) 
\end{align*}

Finally we get $\beta \leq \frac{4}{3}\eps (1+\ln (n/b))$.
\end{proof}

Taking the above into consideration, we parameterize $\beta$ via $\epsilon$ and write $\beta = \beta' \epsilon$. 
This parameterization is correct in the sense that it allows the accuracy parameter $\epsilon$ to appear in our running time bound.
\fi
\bibliography{paper}
\appendix

\end{document}